\documentclass{article}

\PassOptionsToPackage{numbers, compress}{natbib}

\usepackage[final]{neurips_2020}

\usepackage{smile}
\usepackage[utf8]{inputenc} 
\usepackage[T1]{fontenc}    
\usepackage{hyperref}       
\usepackage{url}            
\usepackage{booktabs}       
\usepackage{amsfonts}       
\usepackage{nicefrac}       
\usepackage{microtype}      
\usepackage{cleveref}
\usepackage{color}

\usepackage{mdframed}
\newmdtheoremenv{boxedtheorem}{Theorem}

\newcommand{\la}{\langle}
\newcommand{\ra}{\rangle}
\newcommand{\qvalue}{\Qb}
\newcommand{\vvalue}{\Vb}

\newcommand{\reward}{\rb}

\newcommand{\algname}{model-based EVI}
\newcommand{\Algname}{Model-based EVI}

\usepackage{enumitem}
\usepackage{pgfplots}
\pgfplotsset{compat=1.17}
\usetikzlibrary{arrows,shapes,decorations,automata,backgrounds,petri}

\usepackage[compact]{titlesec}
\usepackage{sidecap}
\titlespacing*{\section}{0pt}{*0.1}{*0.1}
\titlespacing*{\subsection}{0pt}{*0.1}{*0.1}
\titlespacing*{\subsubsection}{0pt}{*0.1}{*0.1}

\allowdisplaybreaks

\title{Provable Multi-Objective Reinforcement Learning with Generative Models}

\author{%
Dongruo Zhou\\
Department of Computer Science\\
University of California, Los Angeles\\
\texttt{drzhou@cs.ucla.edu} \\
\And
Jiahao Chen \\
J.\ P.\ Morgan AI Research \\
New York, NY 10179 \\
\texttt{jiahao.chen@jpmorgan.com} \\
\And
Quanquan Gu \\
Department of Computer Science\\
University of California, Los Angeles\\
\texttt{qgu@cs.ucla.edu} \\
}

\begin{document}

\maketitle

\begin{abstract}
Multi-objective reinforcement learning (MORL) is an extension of ordinary,
single-objective reinforcement learning (RL)
that is applicable to many real-world tasks where multiple objectives exist without known relative costs.
We study the problem of single policy MORL, which learns an optimal policy given the preference of objectives.
Existing methods require strong assumptions such as exact knowledge of the multi-objective Markov decision process,
and are analyzed in the limit of infinite data and time.
We propose a new algorithm called \emph{model-based envelop value iteration (EVI)},
which generalizes the enveloped multi-objective $Q$-learning algorithm in \citet{yang2019generalized}.
Our method can learn a near-optimal value function with polynomial sample complexity and linear convergence speed.
To the best of our knowledge, this is the first finite-sample analysis of MORL algorithms. 
\end{abstract}

\section{Introduction}

Real-world decision-making systems pose many practical challenges for using reinforcement learning (RL) \citep{dulacarnold2019challenges}.
In this paper, we focus on just two.
First, real-world decision-making systems must handle multiple conflicting objectives simultaneously,
yet without obvious preference for any one objective.
For example, a bank may wish to use RL techniques in making credit decisions,
to produce models that adapt changing market structure and account for the historical outcomes of past deals
when making future decisions.
To be profitable, a credit decisioning model will need to consider an applicant's credit risk.
However, the bank also needs to consider other risks, such as reputational risk and counterfactual lost revenue risk associated with falsely declined applications,
and fair lending regulatory risk associated with apparent bias in credit decisions with regard to race, gender, age, or other protected classes \citep{Chen2018,Kurshan2020}.
It is difficult, if not impossible, to assign precise monetary values to these risks,
and therefore the optimal policy for loan approvals cannot be expressed as a straightforward optimization problem to maximize profit.
Ordinary RL algorithms, which work on a single objective function that assigns fixed relative costs to each type of risk,
are therefore unsuitable for these problems.
Generalizations of RL, known as multi-objective reinforcement learning (MORL), have been proposed to address such challenges.
In MORL, the agent has to make decisions not under a single objective, but under multiple objectives,
and can choose different policies flexibly based on different preferences for the objectives.
However, these methods are generally costly and/or require strong assumptions on what is known about the problem.

Second, real-world RL does not have access to the underlying, unknown dynamics of the problem,
thus necessitating learning strategies for the optimal policy that can succeed on finite limited data.
In the credit decisioning example, it is unrealistic to assume that the bank knows perfectly well the outcome of each loan application,
or the precise mechanics of how today's credit needs will affect tomorrow's demands for credit.

Combining these challenges leads to the following question:
\begin{center}
\textit{Can we devise a MORL algorithm with provable finite sampling properties?} 
\end{center}

To the best of our knowledge, no such algorithm currently exists.


\paragraph{Our contributions}
We answer the question above: \textit{yes}.
We propose a new algorithm, which we call model-based envelop value iteration (\algname),
to learn the optimal multi-objective $Q$-function (MOQ).
Our algorithm 
is based on the envelope $Q$-learning algorithm proposed in \citet{yang2019generalized}.
We show that with access to a generative model or simulator,
\algname{} exhibits $\tilde O(mSA/(1-\gamma)^3\epsilon^2)$ sample complexity
and $\tilde O(1/(1-\gamma))$ convergence rate 
to learn an $\epsilon$-suboptimal MOQ function,
where
$m$ is the number of reward functions (objectives),
$S$ is the cardinality of the state space $\cS$,
$A$ is the cardinality of the action space $\cA$,
and $\gamma \in [0,1)$ is the discount factor.
Therefore, the MOQ-learning problem has essentially the same cost as learning an optimal $Q$-function for each objective separately,
and MORL is hence about as complex as $m$ separate RL problems.


\paragraph{Notation}

We use lowercase letters for scalars,
lowercase bold letters for vectors,
and uppercase bold letters for matrices.
For a vector $\xb\in \RR^d$ and matrix $\bSigma\in \RR^{d\times d}$,
we denote by $\|\xb\|_2$ the Euclidean norm
and denote by $\|\xb\|_{\bSigma}=\sqrt{\xb^\top\bSigma\xb}$.
For two sequences $\{a_n\}$ and $\{b_n\}$, we write $a_n=O(b_n)$ if there exists an absolute constant $C$ such that $a_n\leq Cb_n$, and we write $a_n=\Omega(b_n)$ if there exists an absolute constant $C$ such that $a_n\geq Cb_n$. We use $\tilde O(\cdot)$ to further hide the logarithmic factors. 

\section{Related Work}\label{sec:related}

The literature on MORL is relatively sparse compared to the extensive body of work on ordinary RL,
and can be grouped into two main approaches \citep{vamplew2011empirical, roijers2013survey, liu2014multiobjective}.

\paragraph{Multi-policy algorithms}
These methods build and maintain a Pareto-optimal set of optimal policies,
and scale poorly due to the intrinsic growth of the optimality frontier.
\citet{white1982multi} uses dynamic programming to compute a Pareto-optimal set of nonstationary policies.
However, the size of this set increases exponentially with the horizon, making this method impractical.
\citet{barrett2008learning} proposed the convex hull value-iteration method,
which only computes the stationary policies on the convex hull of the Pareto front.
\citet{castelletti2011multi, castelletti2012tree} proposed multi-objective fitted $Q$-iteration (MOFQI),
which construct the $Q$-function approximator with embedded preferences to learn the optimal policy for any given preference during testing.
\citet{wang2013hypervolume} introduced multi-objective Monte-Carlo tree search using the hypervolume indicator \cite{Fleischer2003} to define an action selection criterion that is similar to the upper confidence bound (UCB) in ordinary RL.
The hypervolume indicator is maximized for any policy on the optimality frontier,
but is still expensive to compute,
with the best known practical algorithms requiring a typical complexity of
approximately $\Theta(n^{\log_2 m})$,
where $n$ is the number of Pareto-optimal policies.

\paragraph{Single-policy algorithms}

These methods scalarize the vector of multiple rewards,
collapsing them into a single scalar-valued function using some specification of their preferences,
then apply ordinary RL methods to solve the resulting problem, which is now single-objective.
Single-policy algorithms use less memory and are easier to implement.
However, at each time, a single-policy algorithm finds the optimal policy with respect to some specific preference parameter,
which hinders generalization to other unseen preferences \citep{mannor2004geometric,tesauro2008managing, gabor1998multi, van2013scalarized}.
The simplest of these methods use linear scalarization functions \citep{barrett2008learning, natarajan2005dynamic},
which compute the weighted sum of the values for each objective.
Nonlinear scalarizations have also been proposed \citep{van2013scalarized} to address the limits of the scalarized representation using linear functions. 
\citep{cheung2019regret} studied the regret minimization problem with vectorial feedback and complex objectives, while they need the access to the adapted preference vector. In the next section, we will review the method of \citet{yang2019generalized}, as it forms the starting point for our work.

\section{Preliminaries}

\paragraph{Discounted multi-objective Markov decision processes (MOMDPs)}
We denote a discounted MOMDP by the tuple $(\cS, \cA, \gamma, \reward, \PP, \Omega)$,
where $\cS$ is the state space (possibly infinite),
$\cA$ is the action space,
$\gamma \in [0,1)$ is the discount factor,
$\reward: \cS \times \cA \rightarrow [0,1]^m$ is the vector-valued reward function,
and $m$ is the number of reward functions. For simplicity, we assume the reward function $\reward$ is \emph{deterministic} and \emph{known}. $\PP(s'|s,a) $ is the transition probability function which denotes the probability for state $s$ to transfer to state $s'$ under the action $a$,
$\Omega \subseteq \RR^m$ is the set of preference vectors $\wb\in\Omega$ which represent how to utilize the reward functions. A policy $\pi: \cS \rightarrow \cA$ is a function which maps a state $s$ to an action $a$.
We define the action-value function $\qvalue^{\pi}(s,a)$ and its corresponding value function $\vvalue^{\pi}(s)$ as follows:
\begin{align}
&\qvalue^{\pi}(s,a) = \EE\bigg[\sum_{t = 0}^\infty \gamma^t \reward(s_{t}, a_{t})\bigg|s_0 = s, a_0 = a, \forall t \geq 1,a_{t} = \pi(s_t)  \bigg],\ \vvalue^{\pi}(s) = \qvalue^{\pi}(s,\pi(s)).\notag
\end{align}
We define the optimal value function $\vvalue^*$ and the optimal action-value function $\qvalue^*$ with respect to some weight parameter $\wb \in \RR^m$ as follows:
\begin{align}
    \vvalue^*(s;\wb) = \arg_{\vvalue}\max_{\pi}\wb^\top \vvalue^{\pi}(s),\ \qvalue^*(s,a; \wb) = \arg_{\qvalue}\max_{\pi}\wb^\top\qvalue^{\pi}(s,a),\label{def:optimalQ}
\end{align}
where $\arg_{\vvalue}$ and $\arg_{\qvalue}$ take the vectors of $V$ or $Q$-values that attain the maximum.
For simplicity, we denote $[\PP \vvalue](s,a)=\EE_{s' \sim \PP(\cdot|s,a)}\vvalue(s')$
for any function $\vvalue: \cS\times \Omega \rightarrow \RR^m$.
Therefore, we have the following Bellman equation:
\begin{align}
    \qvalue^{\pi}(s,a) = \reward(s,a) +\gamma\cdot[\PP\vvalue^\pi](s,a).\notag
\end{align}

\paragraph{Problem statement}
In this work, let $\cR$ denote the set of all possible expected returns for some policy given a starting state $s_0$, where
\begin{align}
    \cR: = \bigg\{\qb \in \RR^m:\exists \pi, \qb = \EE\bigg[\sum_{t=0}^\infty \gamma^t \reward(s_t, a_t)\bigg],\ s_t \sim \PP(\cdot|s_{t-1}, a_{t-1}),\ a_t = \pi(s_t)\bigg\},\notag
\end{align}
We aim to find the following possible expected accumulated return $\qb \in \cR$ from a MOMDP belongs to a \emph{Pareto frontier} $\cF^*$, that is $\cF^* := \{\qb \in \cR: \nexists\qb' \in \cR \text{ such that } \qb' \geq \qb \}$. For all preferences in $\Omega$, we define the following \emph{convex converage set (CCS)} of $\cF^*$ as
\begin{align}
    \{\qb \in \cF^*| \exists \wb \in \Omega,\text{such that } \forall \qb' \in \cF^*, \wb^\top\qb \geq \wb^\top \qb'\},\notag
\end{align}
which includes the returns that maximizes the expected accumulated return corresponding to some specific preference $\wb$,
and effectively convexifies the starting Pareto-dominance operator $\geq$ to the operator $\wb^\top\cdot \geq \wb^\top \cdot$.
Our goal is to recover all policies for CCS of any given MOMDP. 

\paragraph{Enveloped $Q$-learning \citep{yang2019generalized}}%
\label{sec:eql}

We now review enveloped $Q$-learning, which was proposed in \citet{yang2019generalized} to solve this problem.
At each time step, enveloped $Q$-learning uses the convex envelope of the solution frontier to update the parameters.
More specifically, the agent initializes the multi-objective $Q$-value function (MOQ) $\qvalue_0(s,a; \wb)$ at the beginning of the algorithm.
At each round $t$, the agent defines the \emph{optimality filter} for any MOQ function $\Qb: \cS \times \cA \times \Omega \rightarrow \RR^m$ as follows:
\begin{align}
[\cH\qvalue](s; \wb) = \arg_{\qvalue}\max_{a \in \cA, \wb' \in \Omega}\wb^\top\qvalue(s,a; \wb'), \label{def:optfilter}
\end{align}
When multiple solutions to \eqref{def:optfilter} exist, it suffices to choose any one solution arbitrarily.

\citet{yang2019generalized} introduced two key concepts, the first being the \emph{multi-objective optimality operator $\cT$},
which is defined following \eqref{def:optfilter} as: 
\begin{align}
\cT\qvalue(s,a;\wb): = \reward(s,a) + \gamma [\PP(\cH\qvalue)](s,a; \wb),\label{def:optbellman}
\end{align}
where $\cT$ does not depend on $\wb$. The second is the following definition of distance between MOQs.
For any two MOQs $\qvalue$ and $\qvalue'$, the distance between them is
\begin{align}
d(\qvalue, \qvalue') = \sup_{(s,a) \in \cS \times \cA, \wb \in \Omega}|\wb^\top \qvalue(s,a; \wb) - \wb^\top \qvalue'(s,a; \wb)|.\label{def:distance}
\end{align}
Due to the nonuniqueness of solutions to \eqref{def:optfilter},
$d(\qvalue, \qvalue') = 0$ does \textit{not} imply that $\qvalue = \qvalue'$,
and thus $d$ is not a true metric since it violates the axiom of identity of indiscernables.
Nevertheless, the $d$ satisfies the weaker axioms of a pseudometric or semimetric, since it is nonnegative, vanishing for all $d(\qvalue,\qvalue)=0$, symmetric, and satisfies the triangle inequality, since
\begin{align}
    d(\qvalue, \qvalue') &= \sup_{(s,a) \in \cS \times \cA, \wb \in \Omega}|\wb^\top \qvalue(s,a; \wb) - \wb^\top \qvalue'(s,a; \wb)|\notag \\
    & \leq \sup_{(s,a) \in \cS \times \cA, \wb \in \Omega}|\wb^\top \qvalue(s,a; \wb) - \wb^\top \qvalue''(s,a; \wb)|\notag \\
    &\quad + \sup_{(s,a) \in \cS \times \cA, \wb \in \Omega}|\wb^\top \qvalue''(s,a; \wb) - \wb^\top \qvalue'(s,a; \wb)|\notag \\
    &  = d(\qvalue, \qvalue'') + d(\qvalue', \qvalue'').\notag
\end{align}

The multi-objective optimality operator $\cT$ has two important properties.

\begin{proposition}[Fixed point \citep{yang2019generalized}]\label{prop:fix}
The optimal $Q$-function $\qvalue^*$ is the fixed point of the multi-objective optimality operator $\cT$, i.e.,
$\qvalue^* = \cT\qvalue^*$.
\end{proposition}

\begin{proposition}[$\gamma$-contraction \citep{yang2019generalized}]\label{prop:concra}
The multi-objective optimality operator $\cT$ is a $\gamma$-contraction operator, where $\gamma$ is the contraction factor. That suggests that the distance between any two MOQs after applying $\cT$ to both of them is less than $\gamma$ times their original distance. In other words, let $\qvalue, \qvalue'$ be any two MOQs, then $d(\cT\qvalue, \cT\qvalue') \leq \gamma d(\qvalue, \qvalue')$. In the context of the discounted MOMDP, the contraction factor is simply the discount factor.
\end{proposition}

Enveloped $Q$-learning presumes that the transition probability function $\PP$ is fully known.
Then, the \emph{envelop value iteration (EVI)} rule suggests that $\qvalue_{t+1} \leftarrow \cT\qvalue_t$.
By Propositions \ref{prop:fix} and \ref{prop:concra}, a generalized form of Banach's fixed-point theorem
yields $\cT^\infty \qvalue = \qvalue^*$ for any MOQ $\qvalue$. 
In contrast, we will now present an alternative algorithm that does not require exact knowledge of the transition probabilities $\PP$,
and show that it has favorable finite sampling properties.

\section{Model-based envelop value iteration (\algname)}

We now present our method, which we call model-based envelop value iteration (\algname{}), in \Cref{algorithm}.
\Algname{} aims to learn a MOQ function which is close to the optimal MOQ function $\qvalue^*$ given a finite number of samples and finite time.
\Algname{} can be divided into two phases: the \emph{data collection phase} and the \emph{evaluation phase}.

\paragraph{Data collection phase}
\Algname{} first aims to collect enough data to learn the unknown transition probability.
We assume that \algname{} has an access to a generative model or simulator,
such that for any state--action pair $(s,a) \in \cS \times \cA$,
\algname{} can sample independent next states $s'$ generated by the underlying transition dynamics.
This is similar to experience replay \citep{mnih2015human} used in many RL applications to collect the samples,
which randomly collects training samples from a batch of previous visited states and actions generated from a stationary distribution.
\Algname{} samples $N$ next states for each state--action pair $(s,a)$,
then builds an empirical transition probability estimate $\hat\PP(\cdot|\cdot,\cdot): \cS \times \cA \times \cS \rightarrow [0,1]$
from these samples:
\begin{align}
\hat \PP(s'|s,a) = \frac{N(s'|s,a)}{N},\label{def:empiricaltransition}
\end{align}
where $N(s'|s,a)$ denotes the number of next states $s'$ sampled starting from $(s,a)$.

\paragraph{Evaluation phase}
Next, \algname{} will learn the optimal MOQ function based on the empirical model $\hat \PP$.
During this phase, \algname{} evaluates the optimal MOQ function based on the empirical model obtained in data collection phase.
At the beginning, \algname{} initializes $\qvalue_0\leftarrow \sum_{n=0}^\infty \gamma^n = 1/(1-\gamma)$, based on a presumed reward of 1.
Similar to the estimated model $\hat \PP$, \algname{} also builds the empirical version of multi-objective optimality operator $\hat\cT$, which is an estimator for $\cT$ in \Cref{def:optbellman} as follows:
\begin{align}
    \hat\cT\qvalue(s,a;\wb):=\reward(s,a) + \gamma [\hat \PP(\cH\qvalue)](s,a; \wb).\label{def:empiricalopt}
\end{align}
Then, \algname{} updates the estimated MOQ function by iteratively applying $\hat\cT$,
obtaining the next MOQ by $\qvalue_{t+1} \leftarrow \hat\cT\qvalue_t$. 
\Cref{def:empiricalopt} is the empirical analogue of \Cref{def:optbellman},
and defines the optimality filter over the empirical model $\hat \PP$
\textit{using a finite number of samples}.
Importantly, we do not need to know the underlying true model $\PP$, which is unacccessible in the practice,
and is in sharp contrast to the requirements of the original envelope $Q$-learning algorithm of \citet{yang2019generalized}
described in \Cref{sec:eql}.

\begin{algorithm}[!ht]
    \caption{Model-based envelop value iteration (\algname{})}
    \label{algorithm}
    \begin{algorithmic}[1]
    \REQUIRE
    State space $\cS$, action space $\cA$, discount factor $\gamma\in(0,1)$, reward function $\reward$, preference vector $\wb\in\Omega$,
    a generative model $\cS \times \cA \rightarrow \cS$ for the next state $s'$ from any state--action pair $(s,a)$,
    and the number of time steps, $T$.
    \STATE Let $\qvalue_0 \leftarrow 1/(1-\gamma)$.
    \STATE For each $(s,a) \in \cS \times \cA$, sample $N$ next states from the generative model.
    \STATE Construct the empirical model for the transition probabilities $\hat \PP$ defined in \eqref{def:empiricaltransition}.
    \STATE Construct the empirical multi-objective optimality operator $\hat\cT$ defined in \eqref{def:empiricalopt}.
    \FOR{each time step $t=1,\dots, T$}
    \STATE Calculate $\qvalue_{t} \leftarrow \hat\cT\qvalue_{t-1}$\label{alg:update}
    \ENDFOR
    \RETURN $\qvalue_T(s,a; \wb)$
    \end{algorithmic}
\end{algorithm}

We now analyze the convergence behavior of \algname{}, which is summarized by the following theorem. 

\begin{boxedtheorem}\label{thm:1}
Suppose that the set of preference vectors $\Omega \subseteq \{\wb: \|\wb\|_1 \leq 1\}$,
and that as $N\uparrow\infty$, the limit $\hat{\mathbb P} \rightarrow \mathbb P$ exists.
Then, there exists some constants $\{c_i\}_{i=1}^4$ such that for all $\epsilon \in(0,1)$ and $\delta \in(0,1)$, if we set the sampling number $N$ and the number of iterations $T$ to be
\begin{align}
    N  = \bigg\lceil \frac{c_1 m}{\epsilon^2(1-\gamma)^3}\log\frac{c_2 mSA}{\delta(1-\gamma)\epsilon}\bigg\rceil,\ T = \bigg\lceil \frac{c_3}{1-\gamma} \log \frac{c_4}{(1-\gamma)\epsilon} \bigg\rceil,  \notag
\end{align}
then with probability at least $1-\delta$, $\qvalue_T$ satisfies $d(\qvalue_T, \qvalue^*) \leq \epsilon$,
where $\qvalue^*$ is the optimal MOQ function as defined in \eqref{def:optimalQ},
and $d$ is the pseudometric over MOQs as defined in \eqref{def:distance}.
\end{boxedtheorem}

In other words, the solution computed by \Cref{algorithm},
while being a solution to the empirical discounted MOMDP
$(\cS, \cA, \gamma, \reward, \hat{\PP}, \Omega)$,
converges to the solution to the true discounted MOMDP
$(\cS, \cA, \gamma, \reward, \PP, \Omega)$,
given sufficiently many samples and time steps.
This is the main result of the paper.

\begin{remark}
Proposition \ref{thm:1} implies that the dependence of the sample complexity on the number of objectives is almost linear ($m\log m$),
which suggests that learning an optimal MOQ function is essentially as hard as to learn these objective functions separately. Meanwhile, the number of iterations $T$ does not depend on $m$, which suggests that MORL is essentially the same as ordinary RL in terms of convergence rate. 
\end{remark}
\begin{remark}
A trivial approach to solve $\qvalue^*(s,a; \wb)$ for any $(s,a)$ and $\wb$ is to enumerate all possible $\wb \in \Omega$ and calculate the optimal MOQ function for each possible $\wb$. However, this naive approach would lead to a dependence on the cardinality of $\Omega$ in the sample complexity since we need to repeat the data collection phase $|\Omega|$ times. In contrast, the sample complexity in Theorem \ref{thm:1} is independent of $|\Omega|$, which suggests that the use of the optimality operator can make MORL more sample efficient. 
\end{remark} 

\begin{remark}
When $m = 1$, the MORL problem degenerates to a single-objective ordinary RL problem, and Proposition \ref{thm:1} suggests a total $NSA = \tilde O(SA/(\epsilon^2(1-\gamma)^3))$ sample complexity and $\tilde O((1-\gamma)^{-1})$ number of iterations. These match the sample complexity and time complexity of \citet{azar2013minimax} for the ordinary RL case. 
\end{remark}

Furthermore, when $m = 1$, the MORL problem degenerates to a single-objective RL problem. Therefore, existing lower bound for RL problem also yields a lower bound for our case. To illustrate the lower bound, we first define the $(\epsilon, \delta)$-correct RL algorithm as follows. 
\begin{definition}[\citet{azar2013minimax}]\label{def:epdelta}
We call an algorithm $\mathbb{A}$ is an $(\epsilon, \delta)$-correct RL algorithm if there exists a class of MOMDPs $M_1, \dots, M_n$ such that with probability at least $1-\delta$, $d(\qvalue^{\mathbb{A}}, \qvalue^*) \leq \epsilon$ holds for all $M_i$, where $\qvalue^{\mathbb{A}}$ is the Q function output by $\mathbb{A}$. 
\end{definition}
By Definition \ref{def:epdelta}, we propose a lower bound of the sample complexity. 
\begin{proposition}[Theorem 3, \citet{azar2013minimax}]
There exist some constants $\{c_i\}_{i=1}^4$ such that for any $\epsilon\in(0, c_1)$, $\delta \in(0,c_2/(SA))$, and for any $(\epsilon,\delta)$-correct RL algorithm $\mathbb{A}$, there exists a tabular MDP $M(\cS, \cA, \gamma, \reward, \PP, \Omega)$ such the total number of samples that $\mathbb{A}$ needs is at least
\begin{align}
\frac{c_3SA}{\epsilon^2(1-\gamma)^3}\cdot \log\frac{c_4SA}{\delta}.\notag
\end{align}

\end{proposition}
\begin{remark}
When $m = 1$, the sample complexity of our algorithm $\tilde O(SA/(\epsilon^2(1-\gamma)^3))$ matches the lower bound, which 
suggests that in general, such a sample complexity can not be improved. We will try to extend the lower bound for single-objective RL to MORL in the future work. 
\end{remark}

\section{Conclusion}
We have proposed a new MORL algorithm, \algname{},
to address two real-world challenges in RL:
multiple objectives with unknown weights, and learning from finite samples.
We show that in order to find an $\epsilon$-suboptimal MOQ function,
it suffices to use $\tilde O(mSA/((1-\gamma)^3\epsilon^2))$ samples and $\tilde O(1/(1-\gamma))$ time steps
as described in \Cref{thm:1},
which implies that learning an optimal MOQ function is essentially as hard as learning $m$ separate objective functions.
Comparing with the lower bound result for single-objective RL suggests that our method is nearly optimal.
We will leave the lower bound of MORL to future work.

\paragraph{Disclaimer}
This paper was prepared for informational purposes in part by the Artificial Intelligence Research group of JPMorgan Chase \& Co and its affiliates (``JP Morgan''), and is not a product of the Research Department of JP Morgan.  JP Morgan makes no representation and warranty whatsoever and disclaims all liability, for the completeness, accuracy or reliability of the information contained herein.  This document is not intended as investment research or investment advice, or a recommendation, offer or solicitation for the purchase or sale of any security, financial instrument, financial product or service, or to be used in any way for evaluating the merits of participating in any transaction, and shall not constitute a solicitation under any jurisdiction or to any person, if such solicitation under such jurisdiction or to such person would be unlawful.

\appendix

\section{Proof of Theorem \ref{thm:1}}
Let $\hat\qvalue^*$ be the optimal action-value function over the empirical transition probability $\hat \PP$. 
We have the following lemmas. 
\begin{lemma}\label{lemma:diff_emp_emp}
We have $d(\qvalue_t, \hat\qvalue^*) \leq \gamma^t /(1-\gamma)$.
\end{lemma}
\begin{proof}[Proof of Lemma \ref{lemma:diff_emp_emp}]
We prove that by induction. For all $t>0$ we have
\begin{align}
    d(\qvalue_t, \hat\qvalue^*) = d(\hat\cT\qvalue_{t-1}, \hat\cT\hat\qvalue^*) \leq \gamma d(\qvalue_{t-1}, \hat\qvalue^*),\label{eq:diff_emp_emp_0}
\end{align}
where the equality holds due to the update rule in Line \ref{alg:update}, Algorithm \ref{algorithm} and the fact $\hat\qvalue^* = \hat\cT\hat\qvalue^*$ by Proposition \ref{prop:fix}, the inequality holds due to Proposition \ref{prop:concra}. Therefore, recursively applying \eqref{eq:diff_emp_emp_0}, we have
\begin{align}
    d(\qvalue_t, \hat\qvalue^*) &\leq \gamma^t d(\qvalue_0, \hat\qvalue^*) \notag \\
    &= \gamma^t \sup_{(s,a) \in \cS \times \cA, \wb \in \Omega}|\wb^\top \qvalue_0(s,a; \wb) - \wb^\top \hat\qvalue^*(s,a; \wb)|\notag \\
    & \leq \gamma^t \sup_{(s,a) \in \cS \times \cA}\|\qvalue_0(s,a; \wb) - \hat\qvalue^*(s,a; \wb)\|_\infty\notag \\
    & \leq\gamma^t/(1-\gamma),\notag
\end{align}
where the second inequality holds due to the fact $\|\wb\|_1 \leq 1$ and Cauchy-Schwarz inequality $\la \ab, \bbb\ra \leq \|\ab\|_\infty \|\bbb\|_1$, and the last inequality holds since $1/(1-\gamma) \one = \qvalue_0(s,a; \wb) \geq \hat\qvalue^*(s,a; \wb) \geq \zero$. That ends our proof. 
\end{proof}

The next lemma provides the upper bound between $\qvalue^*$ and $\hat\qvalue^*$.  
\begin{lemma}\label{lemma:diff_emp_true}
For any $\xi, \delta \in (0,1)$, with probability at least $1-\delta$, we have \begin{align}
    d(\qvalue^*, \hat\qvalue^*) & \leq \sqrt{\frac{4 m\log (8SA/(\xi\delta))}{N(1-\gamma)^3}} + \bigg(\frac{5(\gamma/(1-\gamma)^2)^{4/3}m\log (12SA/(\xi\delta))}{N}\bigg)^{3/4} \notag \\
&\quad + \frac{3m\log (24SA/(\xi\delta))}{(1-\gamma)^3N}+2\xi m/(1-\gamma).\notag
\end{align}
\end{lemma}


\begin{proof}[Proof of Lemma \ref{lemma:diff_emp_true}]
We first show that $\wb^\top \qvalue^*(s,a; \wb)$ is Lipschitz continuous w.r.t. $\wb$. Select $\wb_1, \wb_2 \in \Omega$. Let $\pi_i$ be the optimal policies corresponding to $\wb_i$ satisfying $\wb_i^\top \qvalue^*(s,a; \wb_i) = \wb_i^\top \qvalue^{\pi_i}(s,a)$, $i = 1,2$. Then for any $(s,a) \in \cS \times \cA$, we have
\begin{align}
    &\wb_1^\top \qvalue^*(s,a; \wb_1) - \wb_2^\top \qvalue^*(s,a; \wb_2)\notag \\
    &\quad = \wb_1^\top \qvalue^{\pi_1}(s,a) - \wb_2^\top \qvalue^{\pi_2}(s,a)\notag \\
    &\quad = \wb_1^\top \qvalue^{\pi_1}(s,a) - \wb_2^\top \qvalue^{\pi_1}(s,a) + \wb_2^\top \qvalue^{\pi_1}(s,a) - \wb_2^\top \qvalue^{\pi_2}(s,a)\notag \\
    &\quad \leq \wb_1^\top \qvalue^{\pi_1}(s,a) - \wb_2^\top \qvalue^{\pi_1}(s,a)\notag \\
    &\quad \leq \|\wb_1 - \wb_2\|_\infty \big\|\qvalue^{\pi_1}(s,a) \big\|_1\notag \\
    &\quad \leq \|\wb_1 - \wb_2\|_\infty\cdot m/(1-\gamma),\notag
\end{align}
where the first inequality holds since $\pi_2$ is the optimal policy corresponding to $\wb_2$, the second inequality holds due to Cauchy-Schwarz inequality, the last one holds since $\|\qvalue^{\pi_1}(s,a) \|_1 \leq m\|\qvalue^{\pi_1}(s,a) \|_\infty \leq m/(1-\gamma)$. Similarily we have $\wb_2^\top \qvalue^*(s,a; \wb_2) - \wb_1^\top \qvalue^*(s,a; \wb_1) \leq \|\wb_1 - \wb_2\|_\infty\cdot m/(1-\gamma)$. Therefore, taking maximum over $(s,a) \in \cS \times \cA$, we have
\begin{align}
    \max_{(s,a) \in \cS \times \cA}\big|\wb_1^\top \qvalue^*(s,a; \wb_1) - \wb_2^\top \qvalue^*(s,a; \wb_2)\big| \leq \|\wb_1 - \wb_2\|_\infty\cdot m/(1-\gamma).\label{eq:443}
\end{align}
The same argument also holds for $\hat\qvalue^*(s,a; \wb)$, thus taking maximum over $(s,a) \in \cS \times \cA$, we have
\begin{align}
    \max_{(s,a) \in \cS \times \cA}\big|\wb_1^\top \hat\qvalue^*(s,a; \wb_1) - \wb_2^\top \hat\qvalue^*(s,a; \wb_2)\big| \leq \|\wb_1 - \wb_2\|_\infty\cdot m/(1-\gamma).\label{eq:444}
\end{align}
Let $\cC_{\xi}$ be the $\xi$-covering set of the $\ell_1$ ball with respect to $\ell_\infty$ norm. It is easy to verify that $|\cC_{\xi}| \leq (2/\xi)^m$. For any $\wb \in \cC_\xi$, let $r = \wb^\top\rb$ be the scalar reward function corresponding to the preference $\wb$. Let $Q^*$ be the optimal action-value function with respect to reward function $r$ and transition probability $\PP$, and $\hat Q^*$ be the optimal action-value function with respect to reward function $r$ and transition probability $\hat\PP$. Then we have $r \in [-1,1]$ since $\|\wb\|_1 \leq 1$ and $\|\rb\|_\infty \leq 1$. By Lemma 8,  \citet{azar2013minimax}, with probability at least $1-\delta$, we have \begin{align}
    &\max_{(s,a) \in \cS \times \cA}|Q^*(s,a) - \hat Q^*(s,a)| \notag \\
    &\leq \sqrt{\frac{4 \log (4SA/\delta)}{N(1-\gamma)^3}} + \bigg(\frac{5(\gamma/(1-\gamma)^2)^{4/3}\log (6SA/\delta)}{N}\bigg)^{3/4} + \frac{3\log (12SA/\delta)}{(1-\gamma)^3N}.\label{eq:111}
\end{align}
Meanwhile, note that by the definition of $\qvalue^*$ and $\hat\qvalue^*$, we have
\begin{align}
    Q^*(s,a) = \wb^\top \qvalue^*(s,a; \wb),\ \hat Q^*(s,a) = \wb^\top \hat \qvalue^*(s,a; \wb).\notag
\end{align}
Then substituting the definitions of $Q^*$ and $\hat Q^*$ into \eqref{eq:111} and taking an union bound over all $\wb \in \cC_\xi$, replacing $\delta$ with $\delta/|\cC_\xi|$, we have that with probability at least $1-\delta$, for all $\wb \in \cC_\xi$,
\begin{align}
    &\max_{(s,a) \in \cS \times \cA}\big|\wb^\top \qvalue^*(s,a; \wb) - \wb^\top \hat \qvalue^*(s,a; \wb)\big| \notag \\
    &\leq \sqrt{\frac{4 \log (4SA|\cC_\xi|/\delta)}{N(1-\gamma)^3}} + \bigg(\frac{5(\gamma/(1-\gamma)^2)^{4/3}\log (6SA|\cC_\xi|/\delta)}{N}\bigg)^{3/4} + \frac{3\log (12SA|\cC_\xi|/\delta)}{(1-\gamma)^3N}.\label{eq:333}
\end{align}
Finally, we use the fact that for any $\wb \in \Omega$, there exists $\wb_\xi \in \cC_\xi$ such that $\|\wb - \wb_\xi\|_\infty \leq \xi$. Then with probability at least $1-\delta$, for all $\wb \in \Omega$, we have
\begin{align}
& \max_{(s,a) \in \cS \times \cA}\big|\wb^\top \qvalue^*(s,a; \wb) - \wb^\top \hat \qvalue^*(s,a; \wb)\big| \notag \\
& = \max_{(s,a) \in \cS \times \cA}\big|\wb_\xi^\top \qvalue^*(s,a; \wb_\xi) - \wb_\xi^\top \hat \qvalue^*(s,a; \wb_\xi) + \wb^\top \qvalue^*(s,a; \wb) - \wb_\xi^\top \qvalue^*(s,a; \wb_\xi) \notag \\
&\quad + \wb^\top \hat\qvalue^*(s,a; \wb) - \wb_\xi^\top \hat \qvalue^*(s,a; \wb_\xi)\big| \notag \\
& \leq \max_{(s,a) \in \cS \times \cA}\big|\wb_\xi^\top \qvalue^*(s,a; \wb_\xi) - \wb_\xi^\top \hat \qvalue^*(s,a; \wb_\xi)\big| +\max_{(s,a) \in \cS \times \cA}\big| \wb^\top \qvalue^*(s,a; \wb) - \wb_\xi^\top \qvalue^*(s,a; \wb_\xi)\big| \notag \\
&\quad + \max_{(s,a) \in \cS \times \cA}\big|\wb^\top \hat\qvalue^*(s,a; \wb) - \wb_\xi^\top \hat \qvalue^*(s,a; \wb_\xi)\big| \notag \\
& \leq \sqrt{\frac{4 \log (4SA|\cC_\xi|/\delta)}{N(1-\gamma)^3}} + \bigg(\frac{5(\gamma/(1-\gamma)^2)^{4/3}\log (6SA|\cC_\xi|/\delta)}{N}\bigg)^{3/4} + \frac{3\log (12SA|\cC_\xi|/\delta)}{(1-\gamma)^3N}\notag \\
&\quad + 2\xi m/(1-\gamma),\label{eq:555}
\end{align}
where the first inequality holds due to triangle inequality, the second one holds due to \eqref{eq:443}, \eqref{eq:444}, the fact that $\|\wb - \wb_\xi\|_\infty \leq \xi$ and \eqref{eq:333}. \eqref{eq:555} suggests that with probability at least $1-\delta$, 
\begin{align}
&d(\qvalue^*, \hat\qvalue^*)\notag \\
        & = \max_{(s,a) \in \cS \times \cA, \wb \in \Omega}|\wb^\top \qvalue^*(s,a; \wb) - \wb^\top \hat \qvalue^*(s,a; \wb)| \notag \\
    & \leq \sqrt{\frac{4 \log (4SA|\cC_\xi|/\delta)}{N(1-\gamma)^3}} + \bigg(\frac{5(\gamma/(1-\gamma)^2)^{4/3}\log (6SA|\cC_\xi|/\delta)}{N}\bigg)^{3/4} + \frac{3\log (12SA|\cC_\xi|/\delta)}{(1-\gamma)^3N}\notag \\
&\quad + 2\xi m/(1-\gamma)\notag \\
& \leq \sqrt{\frac{4 m\log (8SA/(\xi\delta))}{N(1-\gamma)^3}} + \bigg(\frac{5(\gamma/(1-\gamma)^2)^{4/3}m\log (12SA/(\xi\delta))}{N}\bigg)^{3/4} + \frac{3m\log (24SA/(\xi\delta))}{(1-\gamma)^3N}\notag \\
&\quad + 2\xi m/(1-\gamma). \notag
\end{align}

\end{proof}
Now we prove Theorem \ref{thm:1}. 
\begin{proof}[Proof of Theorem \ref{thm:1}]
By triangle inequality we have
\begin{align}
    &d(\qvalue_T, \qvalue^*) \notag \\
    &\leq d(\qvalue_T, \hat\qvalue^*) + d(\hat\qvalue^*, \qvalue^*)\notag \\
    & \leq \gamma^T /(1-\gamma)+ \sqrt{\frac{4 m\log (8SA/(\xi\delta))}{N(1-\gamma)^3}} + \bigg(\frac{5(\gamma/(1-\gamma)^2)^{4/3}m\log (12SA/(\xi\delta))}{N}\bigg)^{3/4} \notag \\
&\quad + \frac{3m\log (24SA/(\xi\delta))}{(1-\gamma)^3N}+ 2\xi m/(1-\gamma),\label{eq:222}
\end{align}
where the last inequality holds due to Lemma \ref{lemma:diff_emp_emp} and \ref{lemma:diff_emp_true}. Therefore, set $T = \lceil \log(5/((1-\gamma)\epsilon))/(1-\gamma)\rceil$, $\xi = (1-\gamma)\epsilon/(10m)$, and select $N$ to make sure that
\begin{align}
    \sqrt{\frac{4 m\log (8SA/(\xi\delta))}{N(1-\gamma)^3}} ,\bigg(\frac{5(\gamma/(1-\gamma)^2)^{4/3}m\log (12SA/(\xi\delta))}{N}\bigg)^{3/4} , \frac{3m\log (24SA/(\xi\delta))}{(1-\gamma)^3N} \leq \epsilon/5,\label{eq:888}
\end{align}
we have $d(\qvalue_T, \qvalue^*) \leq \epsilon$. Solving out $N$ ends our proof. 
\end{proof}

\bibliographystyle{plainnat}
\bibliography{reference}

\begin{thebibliography}{21}
\providecommand{\natexlab}[1]{#1}
\providecommand{\url}[1]{\texttt{#1}}
\expandafter\ifx\csname urlstyle\endcsname\relax
  \providecommand{\doi}[1]{doi: #1}\else
  \providecommand{\doi}{doi: \begingroup \urlstyle{rm}\Url}\fi

\bibitem[Azar et~al.(2013)Azar, Munos, and Kappen]{azar2013minimax}
Mohammad~Gheshlaghi Azar, R{\'e}mi Munos, and Hilbert~J Kappen.
\newblock Minimax {PAC} bounds on the sample complexity of reinforcement
  learning with a generative model.
\newblock \emph{Machine Learning}, 91\penalty0 (3):\penalty0 325--349, 2013.
\newblock \doi{10.1007/s10994-013-5368-1}.

\bibitem[Barrett and Narayanan(2008)]{barrett2008learning}
Leon Barrett and Srini Narayanan.
\newblock Learning all optimal policies with multiple criteria.
\newblock In \emph{Proceedings of the 25th International Conference on Machine
  Learning}, ICML '08, pages 41--47, Helsinki, Finland, 2008. Association for
  Computing Machinery.
\newblock \doi{10.1145/1390156.1390162}.

\bibitem[Castelletti et~al.(2011)Castelletti, Pianosi, and
  Restelli]{castelletti2011multi}
Andrea Castelletti, Francesca Pianosi, and Marcello Restelli.
\newblock Multi-objective fitted ${Q}$-iteration: Pareto frontier approximation
  in one single run.
\newblock In \emph{2011 International Conference on Networking, Sensing and
  Control}, ICNSC '11, pages 260--265, Delft, the Netherlands, 2011. IEEE.
\newblock \doi{10.1109/ICNSC.2011.5874921}.

\bibitem[Castelletti et~al.(2012)Castelletti, Pianosi, and
  Restelli]{castelletti2012tree}
Andrea Castelletti, Francesca Pianosi, and Marcello Restelli.
\newblock Tree-based fitted {Q}-iteration for multi-objective {M}arkov decision
  problems.
\newblock In \emph{The 2012 International Joint Conference on Neural Networks},
  IJCNN '12, Brisbane, QLD, Australia, June 2012. IEEE.
\newblock \doi{10.1109/IJCNN.2012.6252759}.

\bibitem[Chen(2018)]{Chen2018}
Jiahao Chen.
\newblock Fair lending needs explainable models for responsible recommendation.
\newblock In \emph{Proceedings of the 2nd FATREC Workshop on Responsible
  Recommendation}, September 2018.

\bibitem[Cheung(2019)]{cheung2019regret}
Wang~Chi Cheung.
\newblock Regret minimization for reinforcement learning with vectorial
  feedback and complex objectives.
\newblock In \emph{Advances in Neural Information Processing Systems}, pages
  726--736, 2019.

\bibitem[Dulac-Arnold et~al.(2019)Dulac-Arnold, Mankowitz, and
  Hester]{dulacarnold2019challenges}
Gabriel Dulac-Arnold, Daniel Mankowitz, and Todd Hester.
\newblock Challenges of real-world reinforcement learning.
\newblock In \emph{Proceedings of the ICML Workshop on Reinforcement Learning
  for Real Life}, June 2019.

\bibitem[Fleischer(2003)]{Fleischer2003}
M.~Fleischer.
\newblock The measure of {P}areto optima: applications to multi-objective
  metaheuristics.
\newblock \emph{Lecture Notes in Computer Science}, 2632:\penalty0 519--533,
  2003.
\newblock \doi{10.1007/3-540-36970-8_37}.

\bibitem[G\'{a}bor et~al.(1998)G\'{a}bor, Kalm\'{a}r, and
  Szepesv\'{a}ri]{gabor1998multi}
Zolt\'{a}n G\'{a}bor, Zsolt Kalm\'{a}r, and Csaba Szepesv\'{a}ri.
\newblock Multi-criteria reinforcement learning.
\newblock In \emph{Proceedings of the Fifteenth International Conference on
  Machine Learning}, ICML '98, pages 197--205, San Francisco, CA, USA, 1998.
  Morgan Kaufmann Publishers Inc.
\newblock \doi{10.5555/645527.657298}.

\bibitem[Kurshan et~al.(2020)Kurshan, Shen, and Chen]{Kurshan2020}
Eren Kurshan, Hongda Shen, and Jiahao Chen.
\newblock Fair lending needs explainable models for responsible recommendation.
\newblock In \emph{Proceedings of the 1st ACM International Conference on
  Artificial Intelligence in Finance}, New York, NY, USA, October 2020. ACM.
\newblock \doi{10.1145/3383455.3422564}.

\bibitem[Liu et~al.(2014)Liu, Xu, and Hu]{liu2014multiobjective}
Chunming Liu, Xin Xu, and Dewen Hu.
\newblock Multiobjective reinforcement learning: A comprehensive overview.
\newblock \emph{IEEE Transactions on Systems, Man, and Cybernetics: Systems},
  45\penalty0 (3):\penalty0 385--398, 2014.
\newblock \doi{10.1109/TSMC.2014.2358639}.

\bibitem[Mannor and Shimkin(2004)]{mannor2004geometric}
Shie Mannor and Nahum Shimkin.
\newblock A geometric approach to multi-criterion reinforcement learning.
\newblock \emph{Journal of Machine Learning Research}, 5:\penalty0 325--360,
  December 2004.

\bibitem[Mnih et~al.(2015)Mnih, Kavukcuoglu, Silver, Rusu, Veness, Bellemare,
  Graves, Riedmiller, Fidjeland, Ostrovski, Petersen, amd Amir~Sadik,
  Antonoglou, King, Kumaran, Wierstra, Legg, and Hassabis]{mnih2015human}
Volodymyr Mnih, Koray Kavukcuoglu, David Silver, Andrei~A Rusu, Joel Veness,
  Marc~G Bellemare, Alex Graves, Martin Riedmiller, Andreas~K Fidjeland, Georg
  Ostrovski, Stig Petersen, Charles~Beattie amd Amir~Sadik, Ioannis Antonoglou,
  Helen King, Dharshan Kumaran, Daan Wierstra, Shane Legg, and Demis Hassabis.
\newblock Human-level control through deep reinforcement learning.
\newblock \emph{Nature}, 518\penalty0 (7540):\penalty0 529--533, 2015.
\newblock \doi{10.1038/nature14236}.

\bibitem[Natarajan and Tadepalli(2005)]{natarajan2005dynamic}
Sriraam Natarajan and Prasad Tadepalli.
\newblock Dynamic preferences in multi-criteria reinforcement learning.
\newblock In \emph{Proceedings of the 22nd International Conference on Machine
  Learning}, ICML '05, pages 601--608, New York, NY, USA, 2005. Association for
  Computing Machinery.
\newblock \doi{10.1145/1102351.1102427}.

\bibitem[Roijers et~al.(2013)Roijers, Vamplew, Whiteson, and
  Dazeley]{roijers2013survey}
Diederik~M Roijers, Peter Vamplew, Shimon Whiteson, and Richard Dazeley.
\newblock A survey of multi-objective sequential decision-making.
\newblock \emph{Journal of Artificial Intelligence Research}, 48:\penalty0
  67--113, 2013.

\bibitem[Tesauro et~al.(2007)Tesauro, Das, Chan, Kephart, Lefurgy, Levine, and
  Rawson]{tesauro2008managing}
Gerald Tesauro, Rajarshi Das, Hoi Chan, Jeffrey~O. Kephart, Charles Lefurgy,
  David~W. Levine, and Freeman Rawson.
\newblock Managing power consumption and performance of computing systems using
  reinforcement learning.
\newblock In \emph{Advances in Neural Information Processing Systems},
  volume~20 of \emph{NIPS'07}, pages 1497--1504, Red Hook, NY, USA, 2007.
  Curran Associates Inc.

\bibitem[Vamplew et~al.(2011)Vamplew, Dazeley, Berry, Issabekov, and
  Dekker]{vamplew2011empirical}
Peter Vamplew, Richard Dazeley, Adam Berry, Rustam Issabekov, and Evan Dekker.
\newblock Empirical evaluation methods for multiobjective reinforcement
  learning algorithms.
\newblock \emph{Machine learning}, 84\penalty0 (1-2):\penalty0 51--80, 2011.

\bibitem[Van~Moffaert et~al.(2013)Van~Moffaert, Drugan, and
  Now{\'e}]{van2013scalarized}
Kristof Van~Moffaert, Madalina~M Drugan, and Ann Now{\'e}.
\newblock Scalarized multi-objective reinforcement learning: Novel design
  techniques.
\newblock In \emph{2013 IEEE Symposium on Adaptive Dynamic Programming and
  Reinforcement Learning}, ADPRL '13, pages 191--199. IEEE, 2013.
\newblock \doi{10.1109/ADPRL.2013.6615007}.

\bibitem[Wang and Sebag(2013)]{wang2013hypervolume}
Weijia Wang and Mich{\`e}le Sebag.
\newblock Hypervolume indicator and dominance reward based multi-objective
  {M}onte-{C}arlo tree search.
\newblock \emph{Machine Learning}, 92\penalty0 (2--3):\penalty0 403--429, 2013.
\newblock \doi{10.1007/s10994-013-5369-0}.

\bibitem[White(1982)]{white1982multi}
D.~J. White.
\newblock Multi-objective infinite-horizon discounted {M}arkov decision
  processes.
\newblock \emph{Journal of Mathematical Analysis and Applications}, 89\penalty0
  (2):\penalty0 639--647, 1982.
\newblock \doi{10.1016/0022-247X(82)90122-6}.

\bibitem[Yang et~al.(2019)Yang, Sun, and Narasimhan]{yang2019generalized}
Runzhe Yang, Xingyuan Sun, and Karthik Narasimhan.
\newblock A generalized algorithm for multi-objective reinforcement learning
  and policy adaptation.
\newblock In H.~Wallach, H.~Larochelle, A.~Beygelzimer, F.~d'Alch\'{e} Buc,
  E.~Fox, and R.~Garnett, editors, \emph{Advances in Neural Information
  Processing Systems}, volume~32, pages 14636--14647, Red Hook, NY, USA, 2019.
  Curran Associates, Inc.

\end{thebibliography}
\end{document}